\documentclass[opre,sglanonrev]{informs4}  % Operations Research

\usepackage{multirow}
\usepackage{booktabs}
\usepackage[linesnumbered,ruled,vlined]{algorithm2e}
\usepackage{placeins}
\definecolor{lightgray}{gray}{0.85}

\OneAndAHalfSpacedXI
\EquationsNumberedThrough    % Default: (1), (2), ...
\TheoremsNumberedThrough     % Preferred (Theorem 1, Lemma 1, Theorem 2)
\ECRepeatTheorems

\newcommand{\Reg}{\operatorname{Reg}}
\definecolor{ZHblue}{RGB}{0,0,0}
\newcommand{\ZH}[1]{\textcolor{black}{#1}}


\begin{document}
%%%%%%%%%%%%%%%%

\RUNAUTHOR{Hu et al.}
\RUNTITLE{Non-stationary Linear Bandits via Misspecification Reductions}

\TITLE{Dynamic Regret for Non-Stationary Linear Bandits via Misspecification Reductions}

\ARTICLEAUTHORS{%
\AUTHOR{Zihao Hu$^{1,3}$, Yuan Yao$^1$, Jiheng Zhang$^{1,3}$, and Zhengyuan Zhou$^2$}
 \AFF{Department of Mathematics, The Hong Kong University of Science and Technology$^1$ \\
 Stern School of Business, New York University$^2$\\
 Department of IEDA, The Hong Kong University of Science and Technology$^3$\\ 
 \EMAIL{\{zihaohu, yuany, jiheng\}@ust.hk, zz26@stern.nyu.edu}}
}

% \ABSTRACT{%
% We study non-stationary stochastic linear and generalized linear bandits under a path-length variation budget. We develop a batchwise CORRAL construction \citep{agarwal2017corralling} that aggregates OFUL-type base learners over a radius-guess grid. The main technical idea is local-norm bias control plus adaptive radius-guess selection: the analysis keeps the non-stationary drift bias in the local geometry of the least-squares estimator, and the algorithm corrals over the corresponding unknown drift scale. When the path length is known, the linear-bandit algorithm achieves dynamic regret $\widetilde{\mathcal O}(d\sqrt T+d^{5/6}T^{2/3}(LP_T)^{1/3})$. The generalized linear extension obtains the same scaling up to link-function constants. We also describe a Bandit-over-Bandit extension that removes prior knowledge of the path length at the standard parameter-free tuning cost.
% }%

\ABSTRACT{%
Many online decision-making problems involve both round-specific feasible
actions and drifting reward models: eligible ad impressions, feasible prices,
and available treatments can change over time, while user preferences,
demand curves, and patient responses may evolve. Motivated by these
applications, we study non-stationary linear bandits with round-specific
feasible decision sets.
{\color{ZHblue}
Our key contribution is a blockwise reduction that connects parameter
drift in non-stationary linear bandits to model misspecification in
contextual linear bandits while preserving the round-specific decision
benchmark. This connection allows us to leverage misspecification-adaptive
learning to obtain dynamic-regret guarantees for general time-varying
decision sets. For a known path-length budget $P_T$, we establish these
guarantees for general compact,
adaptive non-anticipating decision sets under an oblivious parameter
path, and for $K$-armed contextual linear bandits under an oblivious
environment. For general time-varying decision sets, our bounds improve
the dependence on $T$ and $P_T$ from the
$\widetilde O(T^{3/4}P_T^{1/4})$ guarantees available for
existing LinUCB-based approaches to
$\widetilde O(T^{2/3}P_T^{1/3})$, matching the
non-stationary lower-bound dependence up to logarithmic factors.
}
}%


\maketitle
%%%%%%%%%%%%%%%%%%%%%%%%%%%%%%%%%%%%%%%%%%%%%%%%%%%%%%%%%%%%%%%%%%%%%%
\section{Introduction}
Many online decision-making problems, such as online ad display and
contextual treatment assignment, involve choosing among actions represented by feature vectors that
can include both the current context and action-specific information. In these
applications, feasible actions are naturally round-specific: they depend on
the current user or patient and, in ad display, on the ad impressions
currently available. Linear bandits
\citep{abbasi2011improved,chu2011contextual}
provide a standard model for such problems: each feasible decision is
represented by a feature vector, and its expected reward is modeled as the
inner product between this feature vector and an unknown parameter vector.
The standard performance measure for a policy is regret, which compares
the policy's cumulative reward with that of a benchmark. In the classical
stationary model, the unknown parameter is fixed, and this benchmark selects
the best feasible decision under the fixed parameter at each round.

In many applications, however, the relationship between actions and expected
rewards may drift over time: user preferences may evolve, market conditions
may shift, and patient responses to treatments may change. This concern has motivated a broad
literature on non-stationary stochastic optimization and non-stationary
bandit learning
\citep{besbes2011minimax,besbes2015nonstationary,keskin2017chasing,
chen2019nonstationary,cheung2022hedging,wang2025adaptivity}. The non-stationary linear bandit
model captures such drift by allowing the reward parameter to change across
rounds. Formally, consider an online decision-making problem over \(T\) rounds: at
each round \(t\), the learner observes a feasible decision set
\(\mathcal A_t\subseteq\mathbb R^d\), chooses an action
\(a_t\in\mathcal A_t\), and receives a reward with conditional mean
\(\langle a_t,\theta_t\rangle\), where \(\theta_t\) is the current reward
parameter. Following \citet{cheung2022hedging}, the non-stationarity of a
problem instance is measured by the path length of the reward-parameter
sequence, \(P_T:=\sum_{t=1}^{T-1}\|\theta_{t+1}-\theta_t\|_2\).
For a policy \(\pi\), with \(a_t\) denoting its round-\(t\) action, define
the dynamic regret
\[
    \Reg_T(\pi)
    :=
    \mathbb E\sum_{t=1}^T
    \left[
        \max_{a\in\mathcal A_t}\langle a,\theta_t\rangle
        -
        \langle a_t,\theta_t\rangle
    \right],
\]
which compares the learner's expected reward at each round with that of the
best feasible action under the current parameter.  The goal is to achieve small dynamic regret when \(P_T\)
is moderate.
{\color{ZHblue}
Several works study non-stationary linear bandits under path-length
constraints
\citep{russac2019weighted,cheung2022hedging,zhao2021simple,wang2025revisiting}.
\citet{cheung2022hedging} propose a sliding-window method
and introduce the Bandit-over-Bandit (BOB) framework
to adapt to an unknown variation budget.
They also establish the $\Omega(T^{2/3}P_T^{1/3})$
lower-bound dependence.
Under an orthogonal-direction condition on the feasible
actions, their BOB algorithm achieves a
$\widetilde O(T^{2/3}P_T^{1/3}+T^{3/4})$
dynamic-regret guarantee.
When $P_T$ is known, their sliding-window method achieves
a $\widetilde O(T^{2/3}P_T^{1/3})$ guarantee under the
same condition, matching their lower-bound dependence.
This geometric condition is also used by
\citet{faury2021regret} for non-stationary generalized
linear bandits.
Remark~4 of \citet{cheung2022hedging} further describes
an alternative based on restarting adversarial
linear-bandit algorithms, which, for fixed decision sets and
without the orthogonal-direction condition, attains the
$T^{2/3}P_T^{1/3}$ dependence when $P_T$ is known and the
$T^{2/3}P_T^{1/3}+T^{3/4}$ dependence when $P_T$ is unknown.
For a known variation budget $P_T$,
\citet{zhao2021simple} and \citet{wang2025revisiting}
establish dynamic-regret guarantees for restarted and
weighted UCB methods, respectively, with
$\widetilde O(T^{3/4}P_T^{1/4})$ dependence on
the horizon and path length.
Both works also employ the BOB framework to adapt to
an unknown $P_T$ by selecting the restarting period
and discount factor, respectively.
\citet{wei2021non} develop a black-box framework that
achieves near-optimal dynamic regret without prior
knowledge of non-stationarity.
Their linear-bandit instantiation is analyzed for
fixed decision sets.
\par
While linear bandits with fixed decision sets constitute a fundamental
setting, our focus is on general round-specific feasible decision sets,
including adaptive non-anticipating sets under an oblivious parameter path.
Such sets arise naturally in online advertising, where the ads eligible for
display and their associated feature vectors may vary from round to round.
With a fixed decision set and a fixed reward parameter,
the same action can remain optimal throughout a block.
With time-varying decision sets, however, the optimal feasible
action may change from round to round even when the reward
parameter remains fixed.
To address this challenge, our reduction treats parameter
drift within each block as model misspecification while
preserving the benchmark defined by the optimal feasible
action at each round.
\par}

\noindent\textbf{Our contributions.}
%  Our contribution is a misspecification-reduction framework that obtains the
% optimal \(T^{2/3}P_T^{1/3}\) dependence without the orthogonal-direction
% restriction.
{\color{ZHblue}
We establish a reduction from non-stationary linear bandits with
general time-varying decision sets to misspecified linear bandits.
For a known path-length budget, this reduction leverages
misspecification-adaptive guarantees to obtain dynamic-regret
bounds with optimal dependence on $T$ and $P_T$, up to logarithmic
factors. We further establish the conditional analysis needed
to accommodate adaptive non-anticipating decision sets under
an oblivious parameter path.
We instantiate this approach in two settings: general compact
decision sets and $K$-armed contextual linear bandits.
}

The key idea is a local misspecification reduction: parameter drift within
each block is treated as model misspecification, so the desired
dynamic-regret rate follows from blockwise regret guarantees that
\emph{adapt} to the misspecification level. The reduction is based on the
following blockwise bound. On a block
\(\mathcal I=\{\tau,\ldots,\tau+n-1\}\), using \(\theta_\tau\) as the block
comparator gives
\[
    \sup_{t\in\mathcal I}\sup_{a\in\mathcal A_t}
    |\langle a,\theta_t-\theta_\tau\rangle|
    \le
    L P_{\mathcal I},
\]
where \(L\) upper bounds the norms of feasible actions and
\(P_{\mathcal I}\) denotes the path length of the parameter sequence
restricted to block
\(\mathcal I\).  In general, if a block
comparator \(\bar\theta\) has uniform error at most \(\varepsilon\) on
\(\mathcal I\), dynamic regret on the block is bounded by two terms: regret
against the blockwise fixed-parameter benchmark induced by \(\bar\theta\),
and an approximation term \(O(n\varepsilon)\). The first term is precisely
the regret of a misspecified linear bandit on the block, with
misspecification level \(\varepsilon\).
{\color{ZHblue}Suppressing dimension- and norm-dependent factors for this
informal explanation,} a guarantee of
order \(\widetilde O(\sqrt n+n\varepsilon)\) yields the block dynamic-regret
bound \(\widetilde O(\sqrt n+nLP_{\mathcal I})\) after taking
\(\varepsilon=LP_{\mathcal I}\). Restarting on blocks of length
\(\Delta\) and summing over blocks gives the tradeoff
\(\widetilde O(T/\sqrt\Delta+\Delta P_T)\). Choosing the optimal
\(\Delta\) gives the
\(\widetilde O(T^{2/3}P_T^{1/3})\) dependence.

We obtain the following guarantees in two settings:
\begin{itemize}
    \item \textbf{General compact decision sets.}
    For linear bandits with general compact, adaptive non-anticipating
    decision sets and an oblivious parameter path,
    Theorem~\ref{thm:known-path-dynamic-regret} shows that a restarted
    CORRAL aggregation of SquareCB.Lin+ bases
    \citep{foster2020adapting} achieves the expected dynamic-regret bound
    \[
        \widetilde O\!\left(
        d\sqrt T
        +
        L^{1/3}d^{5/6}T^{2/3}P_T^{1/3}
        \right).
    \]
{\color{ZHblue}
    This bound matches the non-stationary lower-bound exponents
    in $T$ and $P_T$, up to logarithmic factors, while allowing
    history-dependent feasible decision sets.
}

    \item \textbf{\(K\)-armed contextual linear bandits.}
    For \(K\)-armed contextual linear bandits under an oblivious adversary,
    with \(\Lambda_K:=1+\log K\),
    Theorem~\ref{thm:known-path-karmed-dynamic-regret} shows that restarting
    SupLinUCB \citep{chu2011contextual} achieves the expected dynamic-regret bound
    \[
        \widetilde O\!\left(
        \sqrt{dT}\,\Lambda_K
        +
        \sqrt d\,\Lambda_K^{5/6}T^{2/3}P_T^{1/3}
        \right).
    \]
    This again matches the optimal dependence on \(T\) and \(P_T\), as
    certified by Proposition~\ref{prop:finite-k-lower-bound}.
\end{itemize}

% The general compact-set result requires an additional conditional argument.
Unlike the original SquareCB.Lin+ misspecification guarantee of
\citet{foster2020adapting}, which considers oblivious sequences, our setting
allows adaptive non-anticipating decision sets. This requires an additional
conditional argument. On each block, the comparator
and misspecification radius are fixed before the within-block randomization.
Consequently, the misspecification-dependent terms remain predictable, and the
SquareCB.Lin+ base and CORRAL master guarantees can be invoked conditionally.
If \(P_T\) is unknown, a Bandit-over-Bandit layer can remove this tuning at
the usual parameter-free cost
\citep{cheung2022hedging,zhao2021simple}.

\noindent\textbf{Related work and positioning.}
{\color{ZHblue}
Our reduction establishes a connection between parameter drift
in non-stationary linear bandits and model misspecification in
contextual linear bandits. The key is to translate within-block
parameter variation into a misspecification bound while retaining
the round-specific contextual benchmark. This connection enables
us to obtain dynamic-regret guarantees with optimal dependence
on $T$ and $P_T$ for general time-varying decision sets when the
path-length budget is known, using the misspecification-adaptive
guarantees of \citet{foster2020adapting,takemura2021parameter}.
For the general compact-set result, we provide the conditional
argument needed to accommodate adaptive non-anticipating
decision sets. For the finite-action setting, we analyze
restarted SupLinUCB under blockwise misspecification.
}

\noindent\textbf{Organization.} The rest of the manuscript is organized as follows. Section~\ref{sec:problem}
introduces the model and block notation. Section~\ref{sec:adaptive-linear}
proves the dynamic-regret guarantee for general linear bandits with
adaptive non-anticipating decision sets. Section~\ref{sec:finite-arm}
develops restarted SupLinUCB and its dynamic-regret guarantee for
\(K\)-armed contextual linear bandits. Section \ref{sec:conc} concludes this manuscript.

\section{Problem Setting}
\label{sec:problem}

We consider a non-stationary contextual linear bandit over
\(T\) rounds.  At each round \(t\), the learner observes a nonempty
feasible decision set \(\mathcal A_t\subseteq\mathbb R^d\).  We identify
each feasible action with its feature vector.
The learner selects an action \(a_t\in\mathcal A_t\) and observes a scalar
reward \(r_t=r_t(a_t)\).
There is an unknown parameter sequence
\(\theta_1,\ldots,\theta_T\in\mathbb R^d\), and the conditional mean
reward is linear: \(\mu_t(a)=\langle a,\theta_t\rangle\) for
\(a\in\mathcal A_t\).
When a norm bound \(S\) on the parameters is imposed, we write
\(\Theta:=\{\theta\in\mathbb R^d:\|\theta\|_2\le S\}\); in the \(K\)-armed
setting below we use the same notation with \(S=1\).

\noindent\emph{Histories.}
Let $\mathcal H_{t-1}$ denote the interaction history through the end of
round $t-1$.  When adaptive decision sets are allowed, we also write
$\mathcal G_t$ for the round-$t$ \(\sigma\)-field after the current decision
set and the learner's action distribution have been determined, but before
the action is sampled and before the reward noise is realized.

\begin{assumption}[Sub-Gaussian reward noise]
    \label{ass:subgaussian-noise}
    For the played action \(a_t\), write
    \(\xi_t:=r_t(a_t)-\mu_t(a_t)\). Conditional on the pre-reward
    information, including \(a_t\), the noise is mean zero and
    \(R\)-sub-Gaussian for a universal constant \(R\).
\end{assumption}

\begin{assumption}[Path length]
    \label{ass:path-length}
    The parameter sequence has path length
    \(P_T:=\sum_{t=1}^{T-1}\|\theta_{t+1}-\theta_t\|_2\).
\end{assumption}

\noindent\emph{Dynamic regret.}
For a policy \(\pi\), let \(a_t\) be the action selected by \(\pi\) at round
\(t\), and let \(a_t^\star\) be an optimal action at round \(t\), i.e.,
\(a_t^\star\in\argmax_{a\in\mathcal A_t}\mu_t(a)
=\argmax_{a\in\mathcal A_t}\langle a,\theta_t\rangle\).
The expected dynamic regret of \(\pi\) is
\[
    \Reg_T(\pi)
    :=
    \mathbb E\left[
        \sum_{t=1}^T
        \left(
            \langle a_t^\star,\theta_t\rangle
            -
            \langle a_t,\theta_t\rangle
        \right)
    \right].
\]
When the policy or algorithm is clear from context, we write \(\Reg_T\) for
\(\Reg_T(\pi)\).

\noindent\emph{\(K\)-armed contextual specialization.}
The \(K\)-armed contextual linear-bandit setting considered by
\citet{chu2011contextual} is the special case
\(\mathcal A_t=\{x_t(i):i\in[K]\}\).  The learner
chooses \(i_t\in[K]\), equivalently \(a_t=x_t(i_t)\), and
\(\mu_t(i):=\mu_t(x_t(i))\).

\noindent\emph{Block notation.}
For an interval \(\mathcal I=\{\tau,\ldots,\tau+n-1\}\), write
\(\mathbb E_\tau[\cdot]:=\mathbb E[\cdot\mid\mathcal H_{\tau-1}]\), where
\(\mathcal H_{\tau-1}\) is the history before the block starts.  The
dynamic regret on block \(\mathcal I\) is $\Reg(\mathcal I)
    :=
    \sum_{t\in\mathcal I}
    \left(
        \langle a_t^\star,\theta_t\rangle
        -
        \langle a_t,\theta_t\rangle
    \right)$.
% \[
%     \Reg(\mathcal I)
%     :=
%     \sum_{t\in\mathcal I}
%     \left(
%         \langle a_t^\star,\theta_t\rangle
%         -
%         \langle a_t,\theta_t\rangle
%     \right).
% \]

\begin{definition}
    \label{def:block}
For a block $\mathcal I=\{\tau,\ldots,\tau+n-1\}$ and a comparator
$\theta\in\Theta$, define
\[
    \varepsilon_{\mathcal I}(\theta)
    :=
    \sup_{t\in\mathcal I}\sup_{a\in\mathcal A_t}
    \left|
        \langle a,\theta_t\rangle
        -
        \langle a,\theta\rangle
    \right|.
\]
Definition \ref{def:block} makes precise the sense in which a non-stationary linear
bandit can be viewed locally as a misspecified stationary linear bandit. When the comparator is clear from context, we write
$\varepsilon_{\mathcal I}$.
\end{definition}      
\section{General Linear Bandits with Adaptive Non-Anticipating Decision Sets}
\label{sec:adaptive-linear}
In this section we consider the setting in which the parameter path
\((\theta_t)_{t=1}^T\) is fixed before interaction begins and is
independent of the learner's randomization and reward noise.  The decision
sets, however, may be adaptive but non-anticipating: at each round \(t\),
\(\mathcal A_t\) may depend on \(\mathcal H_{t-1}\), but it is fixed before
any current-round learner randomization is drawn and before the reward
noise is realized.  Equivalently, after \(\mathcal A_t\) and the learner's
action distribution have been determined, \(\mathcal G_t\) contains these
quantities but not the round-\(t\) action sample.
The technical point, relative to the oblivious-sequence analysis in
\citet{foster2020adapting}, is that our block proof is written
conditionally on this round-\(t\) \(\sigma\)-field, which allows adaptively
generated decision sets.

\ZH{Allowing such adaptivity is important in operational problems where the
feasible actions at a round are shaped by past decisions, resource
states, and observed feedback. Online ad display provides one example. For a
given impression, ad eligibility may depend on past serving decisions through
frequency caps, exposure limits, campaign pacing or delivery constraints, and
advertiser-side resource availability. It may also depend on feedback observed
before the current impression, such as past clicks, conversions, negative
feedback, or other engagement signals. Thus the model allows the feasible
action set to respond to the learner's previous actions and observations,
while requiring it to be fixed before the current reward realization.}

\begin{assumption}[Bounded rewards and parameters for the adaptive reduction]
    \label{ass:bounded-rewards}
    The action vectors and parameters
    satisfy \(\|a\|_2\le L\) for all \(a\in\mathcal A_t\) and
    \(\|\theta_t\|_2\le S\) for all \(t\in[T]\).  Moreover, for every
    round \(t\) and feasible action \(a\in\mathcal A_t\), the potential
    reward satisfies \(0\le r_t(a)\le 1\) a.s.
\end{assumption}

\noindent
Unbounded conditionally sub-Gaussian rewards can be clipped and rescaled,
incurring only standard logarithmic factors. Thus, we state the formal
results under bounded rewards.

\begin{assumption}[Fixed block comparator and fixed radius upper bound]
\label{ass:adaptive-block-comparator}
On block $\mathcal I=\{\tau,\ldots,\tau+n-1\}$, before the learner's
first within-block randomization, there exist quantities
$\theta_{\mathcal I}^\star\in\Theta$ and
$\bar\varepsilon_{\mathcal I}\ge0$, unknown to the learner but fixed, such
that \(\varepsilon_{\mathcal I}(\theta_{\mathcal I}^\star)
\le \bar\varepsilon_{\mathcal I}\) almost surely. In the path-length
application below, this assumption is satisfied by taking
\(\theta_{\mathcal I}^\star=\theta_\tau\) and
\(\bar\varepsilon_{\mathcal I}=L P_{\mathcal I}\), since the parameter path
is fixed before interaction begins.
\end{assumption}

\noindent\emph{Algorithmic idea.}
Following the restarting viewpoint of
\citet{besbes2015nonstationary}, the algorithm partitions the horizon into
blocks and restarts a CORRAL master
\citep{agarwal2017corralling,foster2020adapting} over SquareCB.Lin+ bases.
The CORRAL layer randomizes over base learners and competes with the best
base learner in hindsight, up to its aggregation cost. In our use of CORRAL,
the bases are indexed by a geometric grid of candidate misspecification
levels. On each block, the unknown parameter drift acts as an unknown
misspecification level, and the grid contains a candidate within a constant
factor of this level.

\begin{algorithm}[!htbp]
    \caption{Restarted misspecification-adaptive linear bandit}
    \label{alg:restart-misspec-adaptive}
    % \begin{algorithmic}[1]
    \KwIn{Horizon $T$, block length $\Delta$, dimension $d$, radius cap $B_\varepsilon=2LS$, square-loss regret bound $\operatorname{Reg}_{\rm sq}(\Delta)$}
    \KwOut{Actions $a_t$}
    
    Partition time into blocks $\mathcal I$ of length $\Delta$\;
    
    \ForEach{block $\mathcal I$}{
        Initialize a geometric grid
        $\{\varepsilon'_m\}_{m=1}^M\subseteq[1/|\mathcal I|,B_\varepsilon]$
        with endpoints $B_\varepsilon$ and $1/|\mathcal I|$, up to a factor two\;
        Initialize a hedged-FTRL/CORRAL-style master over $M$ base learners\;
        Initialize $M$ base learners, each a SquareCB.Lin+ instance\;
    
        \For{$t \in \mathcal I$}{
            Observe $\mathcal A_t$\;
            Master samples $M_t \sim q_t$\;
            Set
            $\rho_{t,M_t}=\frac{1}{\min_{\tau\le r\le t} q_{r,M_t}}$\;
            Base learner $M_t$ sets
            $\gamma_{t,M_t}=\min\left\{
            \frac{\sqrt d}{\varepsilon'_{M_t}},
            \sqrt{\frac{d\Delta}{\rho_{t,M_t}\operatorname{Reg}_{\rm sq}(\Delta)}}
            \right\}$\;
            Set oracle weight
            $w_t=\frac{\gamma_{t,M_t}}{q_{t,M_t}}$\;
            Query the weighted regression oracle of base $M_t$, implemented via the randomized weighted-update reduction \citep{foster2020adapting}, to obtain $\widehat\beta_{t,M_t}$\;
            Set
            $p_{t,M_t}\in \operatorname{logdet\text{-}barrier}
            \bigl(\widehat\beta_{t,M_t},\gamma_{t,M_t};\mathcal A_t\bigr)$\;
            Sample $a_t\sim p_{t,M_t}$\;
            Play $a_t$, observe reward $r_t$, and set $\ell_t=-r_t$\;
            Update the weighted oracle of base $M_t$ with $(w_t,a_t,\ell_t)$\;
            Update the master using the shifted observed loss $\ell_t+1=1-r_t$ of base $M_t$\;
        }
    }
    \end{algorithm}

We now describe the grid, oracle interface, and sampling rule used in
Algorithm~\ref{alg:restart-misspec-adaptive}.
On a block of length \(n\), let \(B_\varepsilon:=2LS\) and use a geometric
grid \(\{\varepsilon'_m\}_{m=1}^M\subseteq[1/n,B_\varepsilon]\), with
largest point \(B_\varepsilon\), smallest point \(1/n\), and common ratio
two. For base \(m\), set the exploration parameter
\[
    \gamma_{t,m}
    =
    \min\left\{
        \frac{\sqrt d}{\varepsilon'_m},
        \sqrt{\frac{dn}{\rho_{t,m}\Reg_{\rm sq}(n)}}
    \right\}.
\]
Although Algorithm~\ref{alg:restart-misspec-adaptive} invokes a weighted
square-loss oracle, the randomized reduction of
\citet{foster2020adapting} allows this interface to be implemented using
an unweighted online square-loss regression oracle. For the bounded
linear square-loss class considered here, the online Newton step of
\citet{hazan2007logarithmic} gives
\(\Reg_{\rm sq}(n)=\widetilde O(d)\).

For the SquareCB.Lin+ sampling rule, let \(\Delta(\mathcal A)\) be the
set of distributions on \(\mathcal A\), and for
\(p\in\Delta(\mathcal A)\) write
\(\bar a_p:=\mathbb E_{a\sim p}[a]\) and
\(H_p:=\mathbb E_{a\sim p}[aa^\top]\).  Given
\((\widehat\beta,\gamma,\mathcal A)\), the base chooses any minimizer of
\begin{equation}
\label{eq:logdet-barrier-definition}
\begin{aligned}
    \argmin_{p\in\Delta(\mathcal A)}
    \Bigl\{\langle \bar a_p,\widehat\beta\rangle-\gamma^{-1}\log\det(H_p-\bar a_p\bar a_p^\top)
    \Bigr\}.
\end{aligned}
\end{equation}
We write
\(\operatorname{logdet\text{-}barrier}(\widehat\beta,\gamma;\mathcal A)\)
for this set of minimizers.

At round \(t\), the master samples an index \(M_t\sim q_t\) and follows
base learner \(M_t\).  Let \(Z_{t,m}:=\mathbf 1\{M_t=m\}\).  On a block
\(\mathcal I=\{\tau,\ldots,\tau+n-1\}\), define
\(\rho_{t,m}:=1/\min_{\tau\le r\le t}q_{r,m}\) and
\(\rho_{\mathcal I,m}:=\max_{t\in\mathcal I}\rho_{t,m}\).  The master
outputs an interior distribution, so \(q_{t,m}>0\).  If base \(m\) is
selected, its oracle weight is \(w_{t,m}:=\gamma_{t,m}/q_{t,m}\).  Each
base \(m\) proposes an action distribution \(p_{t,m}\). The selected base
therefore proposes \(p_{t,M_t}\), and the learner samples
\(a_t\sim p_{t,M_t}\).

Lemma~\ref{thm:block-mis-adaptive} gives the block regret guarantee for
Algorithm~\ref{alg:restart-misspec-adaptive} under the fixed-radius
misspecification condition above.
\begin{lemma}
    \label{thm:block-mis-adaptive}
    Assume Assumptions~\ref{ass:subgaussian-noise},~\ref{ass:bounded-rewards}
    and~\ref{ass:adaptive-block-comparator} hold on block \(\mathcal I\).
    Run Algorithm~\ref{alg:restart-misspec-adaptive} freshly on
    \(\mathcal I\), using observed loss \(\ell_t=-r_t(a_t)\) and updating
    the master with the shifted loss \(\ell_t+1\). Assume the unweighted
    square-loss oracle has regret \(\Reg_{\rm sq}(n)\) on length-\(n\)
    adaptive non-anticipating sequences. Then
    \[
        \mathbb E_\tau[\Reg(\mathcal I)]
        \le
        \widetilde O\!\left(
            \sqrt{dn\,\Reg_{\rm sq}(n)}
            +
            \sqrt d\,n\,\bar\varepsilon_{\mathcal I}
        \right).
    \]
    In particular, online Newton step gives
    $\Reg_{\rm sq}(n)=\widetilde O(d)$, and hence
    \[
        \mathbb E_\tau[\Reg(\mathcal I)]
        \le
        \widetilde O\!\left(
            d\sqrt n
            +
            \sqrt d\,n\,\bar\varepsilon_{\mathcal I}
        \right).
    \]
    \end{lemma}
\begin{proof}{Proof of Lemma~\ref{thm:block-mis-adaptive}.}
Let \(\mathcal G_t\) be the round-\(t\) \(\sigma\)-field defined above.
Let \(\ell_t=-r_t(a_t)\) denote the observed loss supplied to the selected
base. The master receives the shifted loss \(\ell_t+1=1-r_t(a_t)\).  Under
the mean-loss convention \(L_t(a)=-\langle a,\theta_t\rangle\), we have
\(\mathbb E[\ell_t\mid \mathcal G_t,M_t,a_t]=L_t(a_t)\).
Then \(a_t^\star\in\argmax_{a\in\mathcal A_t}
\langle a,\theta_t\rangle
=\argmin_{a\in\mathcal A_t}L_t(a)\),
and therefore
\[
    L_t(a_t)-L_t(a_t^\star)
    =
    \langle a_t^\star,\theta_t\rangle
    -
    \langle a_t,\theta_t\rangle .
\]
Thus the reward-regret on the block is exactly the loss-regret for the
losses $L_t$.

We condition on \(\mathcal F_{\mathcal I}\), the \(\sigma\)-field just before
the first within-block randomization.
Under Assumption~\ref{ass:adaptive-block-comparator}, the comparator
$\theta_{\mathcal I}^\star$ is fixed under this conditioning.  Put
$\beta_{\mathcal I}^\star:=-\theta_{\mathcal I}^\star$. Then, for every
$t\in\mathcal I$,
\begin{equation}
\begin{aligned}
&\sup_{a\in\mathcal A_t}
    \left|
        L_t(a)-\langle a,\beta_{\mathcal I}^\star\rangle
    \right|\\
=&
\sup_{a\in\mathcal A_t}
    \left|
        \langle a,\theta_t\rangle
        -
        \langle a,\theta_{\mathcal I}^\star\rangle
    \right|\le
\bar\varepsilon_{\mathcal I}.
\end{aligned}
\end{equation}
The proof follows the structure of the oblivious-sequence argument of
\citet{foster2020adapting}, with the sequence-level comparator and
misspecification radius replaced by the fixed block pair
\((\beta_{\mathcal I}^\star,\bar\varepsilon_{\mathcal I})\).
For each \(t\in\mathcal I\), the residual
\[
    e_t(a):=L_t(a)-\langle a,\beta_{\mathcal I}^\star\rangle
\]
is \(\mathcal G_t\)-measurable as a function of
\(a\in\mathcal A_t\). Moreover,
\(\sup_{a\in\mathcal A_t}|e_t(a)|\le\bar\varepsilon_{\mathcal I}\).

The learner's regret under CORRAL decomposes into two parts: master regret
and base regret. The base-regret term for base \(m\) takes the following
importance-weighted form:
\[
    \Reg_{\rm Imp}^m(\mathcal I)
    :=
    \mathbb E_\tau\!\left[
        \sum_{t\in\mathcal I}
        \frac{Z_{t,m}}{q_{t,m}}
        \bigl(
            L_t(a_t)-L_t(a_t^\star)
        \bigr)
    \right].
\]
Here \(Z_{t,m}/q_{t,m}\) is the importance weight that makes the selected
rounds for base \(m\) an unbiased conditional estimate of the loss-regret of
the counterfactual actions proposed by that base.

For every bounded random function \(\phi_t\) on
\(\mathcal A_t\) whose values are \(\mathcal G_t\)-measurable,
\begin{equation}
    \label{eq:conditional-importance-weighting}
    \mathbb E\!\left[
        \frac{Z_{t,m}}{q_{t,m}}\phi_t(a_t)
        \,\middle|\,\mathcal G_t
    \right]
    =
    \mathbb E_{a\sim p_{t,m}}[\phi_t(a)] .
\end{equation}
Because \(\beta_{\mathcal I}^\star\) and
\(\bar\varepsilon_{\mathcal I}\) are fixed under the block conditioning,
the residual-dependent functions used below are \(\mathcal G_t\)-measurable,
so \eqref{eq:conditional-importance-weighting} applies.

Let \(a_t^\circ\in\argmin_{a\in\mathcal A_t}
\langle a,\beta_{\mathcal I}^\star\rangle\) be an action optimal under the
fixed block comparator.  We use
deterministic tie-breaking so that \(a_t^\circ\) is
\(\mathcal G_t\)-measurable.
Next we present the key conversion step: it upper bounds instantaneous
dynamic regret against the true roundwise optimum by instantaneous static
regret under the fixed block model, plus the block misspecification
penalty.
By the block misspecification bound,
\[
\begin{aligned}
    L_t(a_t)-L_t(a_t^\star)
    &\le
    \langle a_t,\beta_{\mathcal I}^\star\rangle
    -
    \langle a_t^\star,\beta_{\mathcal I}^\star\rangle
    +
    2\bar\varepsilon_{\mathcal I}  \\
    &\le
    \langle a_t-a_t^\circ,
    \beta_{\mathcal I}^\star\rangle
    +
    2\bar\varepsilon_{\mathcal I}.
\end{aligned}
\]
For a distribution \(p\) on \(\mathcal A_t\), write
\(\bar a_p:=\mathbb E_{a\sim p}[a]\). Thus \(\bar a_{p_{t,m}}\) is the
mean action proposed by base \(m\) at round \(t\).
Therefore
\[
\begin{aligned}
    \Reg_{\rm Imp}^m(\mathcal I)
    &\le
    \mathbb E_\tau\!\left[
        \sum_{t\in\mathcal I}
        \frac{Z_{t,m}}{q_{t,m}}
        \langle a_t-a_t^\circ,
        \beta_{\mathcal I}^\star\rangle
    \right]
    +
    2\bar\varepsilon_{\mathcal I}n  \\
    &=
    \mathbb E_\tau\!\left[
        \sum_{t\in\mathcal I}
        \langle \bar a_{p_{t,m}}-a_t^\circ,
        \beta_{\mathcal I}^\star\rangle
    \right]
    +
    2\bar\varepsilon_{\mathcal I}n .
\end{aligned}
\]
The equality uses \eqref{eq:conditional-importance-weighting} with
\(\phi_t(a)=\langle a-a_t^\circ,
\beta_{\mathcal I}^\star\rangle\) and \(\phi_t\equiv1\). Applying
\eqref{eq:conditional-importance-weighting} with \(\phi_t\equiv1\) gives
\(\mathbb E[Z_{t,m}/q_{t,m}\mid\mathcal G_t]=1\), so the additive
misspecification term is bounded by \(2\bar\varepsilon_{\mathcal I}n\).

The base uses
\[
    \gamma_{t,m}
    =
    \min\left\{
        \frac{\sqrt d}{\varepsilon'_m},
        \sqrt{\frac{dn}{\rho_{t,m}\Reg_{\rm sq}(n)}}
    \right\}.
\]
Let \(\widehat\beta_{t,m}\) be the \(\mathcal G_t\)-measurable prediction
made by base \(m\)'s square-loss oracle at round \(t\).  Thus
\(\langle a,\widehat\beta_{t,m}\rangle\) is the base's predicted loss for
action \(a\).
By Lemma~\ref{lem:logdet-one-step}, applied conditionally on
\(\mathcal G_t\), we have
\[
\begin{aligned}
    &\langle \bar a_{p_{t,m}}-a_t^\circ,
    \beta_{\mathcal I}^\star\rangle\\
    \le&
    \frac{d}{\gamma_{t,m}}
    +
    \frac{\gamma_{t,m}}{4}
    \|\widehat\beta_{t,m}-\beta_{\mathcal I}^\star
    \|_{\mathbb E_{a\sim p_{t,m}}[aa^\top]}^2 .
\end{aligned}
\]
Let \(v_{t,m}:=\widehat\beta_{t,m}-\beta_{\mathcal I}^\star\).  The norm
term is a second moment under the action distribution proposed by base
\(m\):
\[
    \|v_{t,m}\|_{\mathbb E_{a\sim p_{t,m}}[aa^\top]}^2
    =
    \mathbb E_{a\sim p_{t,m}}
    \!\left[\langle a,v_{t,m}\rangle^2\right].
\]
Applying \eqref{eq:conditional-importance-weighting} with
\(\phi_t(a)=\gamma_{t,m}\langle a,v_{t,m}\rangle^2\), and then summing over
\(t\in\mathcal I\), rewrites this distributional second moment as the
importance-weighted realized square term
\[
    D_{\mathcal I,m}
    :=
    \mathbb E_\tau\!\left[
        \sum_{t\in\mathcal I}
        \frac{Z_{t,m}}{q_{t,m}}
        \gamma_{t,m}
        \langle a_t,\widehat\beta_{t,m}
        -\beta_{\mathcal I}^\star\rangle^2
    \right].
\]
Then
\begin{equation}
    \label{eq:base-regret-before-square-loss}
    \Reg_{\rm Imp}^m(\mathcal I)
    \le
    \sum_{t\in\mathcal I}
    \mathbb E_\tau\!\left[\frac{d}{\gamma_{t,m}}\right]
    +
    \frac14D_{\mathcal I,m}
    +
    2\bar\varepsilon_{\mathcal I}n .
\end{equation}

It remains to relate \(D_{\mathcal I,m}\) to weighted square-loss regret.
The oracle weight conditional on selecting base \(m\) is
\(w_{t,m}=\gamma_{t,m}/q_{t,m}\).  Equivalently, the effective weight for
the length-\(n\) block sequence is
\[
    \omega_{t,m}:=\frac{Z_{t,m}\gamma_{t,m}}{q_{t,m}} .
\]
On every round for which base learner \(m\) is selected,
\begin{equation}
\label{eq:weighted-square-decomposition}
    \begin{aligned}
    &\langle a_t,\widehat\beta_{t,m}
        -\beta_{\mathcal I}^\star\rangle^2 \\
    =&
    (\langle a_t,\widehat\beta_{t,m}\rangle-\ell_t)^2
    -
    (\langle a_t,\beta_{\mathcal I}^\star\rangle-\ell_t)^2 \\
    &\quad
    +
    2\bigl(\ell_t-\langle a_t,\beta_{\mathcal I}^\star\rangle\bigr)
    \langle a_t,\widehat\beta_{t,m}
        -\beta_{\mathcal I}^\star\rangle .
    \end{aligned}
\end{equation}
The weighted-update reduction of \citet{foster2020adapting}, applied to the
underlying unweighted square-loss oracle, gives
\begin{equation}
\label{eq:weighted-square-excess-bound}
    \begin{aligned}
    &\mathbb E_\tau\!\left[
        \sum_{t\in\mathcal I}
    \omega_{t,m}\left((\langle a_t,\widehat\beta_{t,m}\rangle-\ell_t)^2
    -
    (\langle a_t,\beta_{\mathcal I}^\star\rangle-\ell_t)^2\right)
    \right] \\
    &\qquad\le
    \mathbb E_\tau\!\left[
    \max_{t\in\mathcal I}
        \frac{\gamma_{t,m}}{q_{t,m}}
    \right]\Reg_{\rm sq}(n).
    \end{aligned}
\end{equation}
For the remaining cross-term, write
\(e_t(a):=L_t(a)-\langle a,\beta_{\mathcal I}^\star\rangle\). Since
\(\mathbb E[\ell_t\mid\mathcal G_t,M_t,a_t]=L_t(a_t)\), the
stochastic noise part vanishes by the tower property. Hence
\begin{equation}
\label{eq:weighted-cross-term-bound}
    \begin{aligned}
    &2\mathbb E_\tau\!\left[
        \sum_{t\in\mathcal I}
    \omega_{t,m}
    \bigl(\ell_t-\langle a_t,\beta_{\mathcal I}^\star\rangle\bigr)
    \langle a_t,\widehat\beta_{t,m}
        -\beta_{\mathcal I}^\star\rangle
    \right] \\
    =&
    2\mathbb E_\tau\!\left[
    \sum_{t\in\mathcal I}
    \omega_{t,m}
    e_t(a_t)
    \langle a_t,\widehat\beta_{t,m}
        -\beta_{\mathcal I}^\star\rangle
    \right]  \\
    \le&
    2\mathbb E_\tau\!\left[
    \sum_{t\in\mathcal I}
    \omega_{t,m}e_t(a_t)^2
    \right]
    +
    \frac12D_{\mathcal I,m},
    \end{aligned}
\end{equation}
where the last step uses \(2uv\le 2u^2+\frac12v^2\). The first term is
controlled pointwise by the block misspecification:
\begin{equation}
\label{eq:weighted-residual-square-bound}
    \begin{aligned}
    &\mathbb E_\tau\!\left[
        \sum_{t\in\mathcal I}
        \omega_{t,m}e_t(a_t)^2
    \right]\\
    =&
    \mathbb E_\tau\!\left[
        \sum_{t\in\mathcal I}
        \gamma_{t,m}
        \mathbb E_{a\sim p_{t,m}}[e_t(a)^2]
    \right] \\
    \le&
    \mathbb E_\tau\!\left[
        \sum_{t\in\mathcal I}
        \gamma_{t,m}\bar\varepsilon_{\mathcal I}^2
    \right]  \le
    n\frac{\sqrt d}{\varepsilon'_m}
    \bar\varepsilon_{\mathcal I}^2 .
    \end{aligned}
\end{equation}
Combining
\eqref{eq:weighted-square-decomposition},
\eqref{eq:weighted-square-excess-bound},
\eqref{eq:weighted-cross-term-bound}, and
\eqref{eq:weighted-residual-square-bound}, and then isolating
\(D_{\mathcal I,m}\), yields
\begin{equation}
\label{eq:weighted-D-bound}
    D_{\mathcal I,m}
    \le
    2\mathbb E_\tau\!\left[
        \max_{t\in\mathcal I}
        \frac{\gamma_{t,m}}{q_{t,m}}
    \right]\Reg_{\rm sq}(n)
    +
    4n\frac{\sqrt d}{\varepsilon'_m}
    \bar\varepsilon_{\mathcal I}^2 .
\end{equation}
Substituting \eqref{eq:weighted-D-bound} into
\eqref{eq:base-regret-before-square-loss} gives
\[
\begin{aligned}
    &\Reg_{\rm Imp}^m(\mathcal I)\\
    \le&
    \sum_{t\in\mathcal I}
    \mathbb E_\tau\!\left[\frac{d}{\gamma_{t,m}}\right]
    +
    \frac12
    \mathbb E_\tau\!\left[
        \max_{t\in\mathcal I}
        \frac{\gamma_{t,m}}{q_{t,m}}
    \right]\Reg_{\rm sq}(n)  \\
    &\qquad
    +
    n\frac{\sqrt d}{\varepsilon'_m}
    \bar\varepsilon_{\mathcal I}^2
    +
    2\bar\varepsilon_{\mathcal I}n .
\end{aligned}
\]
Since
\[
    \frac1{\gamma_{t,m}}
    \le
    \frac{\varepsilon'_m}{\sqrt d}
    +
    \sqrt{
        \frac{\rho_{\mathcal I,m}\Reg_{\rm sq}(n)}{dn}
    },
\]
we have
\[
\begin{aligned}
    &\sum_{t\in\mathcal I}
    \mathbb E_\tau\!\left[\frac{d}{\gamma_{t,m}}\right]\\
    \le&
    \varepsilon'_m n\sqrt d
    +
    \mathbb E_\tau[\sqrt{\rho_{\mathcal I,m}}]
    \sqrt{dn\,\Reg_{\rm sq}(n)}.
\end{aligned}
\]
Moreover,
\[
    \max_{t\in\mathcal I}
    \frac{\gamma_{t,m}}{q_{t,m}}
    \le
    \max_{t\in\mathcal I}
    \gamma_{t,m}\rho_{t,m}
    \le
    \sqrt{
        \frac{\rho_{\mathcal I,m}dn}{\Reg_{\rm sq}(n)}
    }.
\]
Therefore
\[
\begin{aligned}
    \Reg_{\rm Imp}^m(\mathcal I)
    &\le
    \frac{3}{2}
    \mathbb E_\tau[\sqrt{\rho_{\mathcal I,m}}]
    \sqrt{dn\,\Reg_{\rm sq}(n)}  \\
    &\qquad
    +
    \left(
        \varepsilon'_m
        +
        \frac{\bar\varepsilon_{\mathcal I}^2}{\varepsilon'_m}
    \right)n\sqrt d
    +
    2\bar\varepsilon_{\mathcal I}n .
\end{aligned}
\]
It remains to combine the base guarantee with the master guarantee.  For
each base learner \(m\), let \(a_{t,m}\) denote the action drawn from
\(p_{t,m}\) at time \(t\).  The master is updated with the shifted observed
loss \(\ell_t+1\), because \(\ell_t=-r_t(a_t)\) may be negative whereas the
master guarantee (Lemma~\ref{lem:hedged-tsallis-master}) requires
nonnegative bounded losses. After conditioning on \(\mathcal G_t\), the
mean shifted loss associated with base learner \(m\) is
\(\mathbb E_{a\sim p_{t,m}}[L_t(a)+1]\). The same shift is applied to every
base learner, so it cancels in the master regret comparison. For any fixed
base \(m\),
\[
\begin{aligned}
    &\mathbb E_\tau[\Reg(\mathcal I)]\\
    =&
    \mathbb E_\tau\!\left[
        \sum_{t\in\mathcal I}
        \bigl(
            L_t(a_t)-L_t(a_t^\star)
        \bigr)
    \right] \\
    =&
    \mathbb E_\tau\!\left[
        \sum_{t\in\mathcal I}
        \bigl(
            (L_t(a_t)+1)
            -
            (L_t(a_{t,m})+1)
        \bigr)
    \right]\\
    &+
    \Reg_{\rm Imp}^m(\mathcal I).
\end{aligned}
\]
The equality follows from \eqref{eq:conditional-importance-weighting}:
\[
    \mathbb E_\tau[L_t(a_{t,m})+1]
    =
    \mathbb E_\tau\!\left[
        \frac{Z_{t,m}}{q_{t,m}}
        \bigl(L_t(a_t)+1\bigr)
    \right].
\]

Apply Lemma~\ref{lem:hedged-tsallis-master} to the shifted master losses
\(\tilde L_{t,m}=L_t(a_{t,m})+1\in[0,1]\subset[0,2]\), we have
\[
\begin{aligned}
    &\mathbb E_\tau\!\left[
        \sum_{t\in\mathcal I}
        \bigl(L_t(a_t)-L_t(a_{t,m})\bigr)
    \right]\\
    \le&
    \widetilde O\!\left(
        \sqrt{dn\,\Reg_{\rm sq}(n)}
    \right)\\
    &-
    \frac{3}{2}
    \mathbb E_\tau[\sqrt{\rho_{\mathcal I,m}}]
    \sqrt{dn\,\Reg_{\rm sq}(n)} .
\end{aligned}
\]
Adding the preceding base bound cancels the
$\rho_{\mathcal I,m}$-dependent terms:
\[
\begin{aligned}
    \mathbb E_\tau[\Reg(\mathcal I)]
    &\le
    \widetilde O\!\left(
        \sqrt{dn\,\Reg_{\rm sq}(n)}
    \right)  \\
    &\qquad
    +
    \left(
        \varepsilon'_m
        +
        \frac{\bar\varepsilon_{\mathcal I}^2}{\varepsilon'_m}
    \right)n\sqrt d
    +
    2\bar\varepsilon_{\mathcal I}n .
\end{aligned}
\]

Finally choose the analysis grid point \(m^\star\) using the fixed upper
bound \(\bar\varepsilon_{\mathcal I}\). Such a point always exists because
the geometric grid covers \([1/n,B_\varepsilon]\) up to a factor two and
has \(1/n\) as its smallest point. If
\(\bar\varepsilon_{\mathcal I}\ge 1/n\), the grid construction gives
\(m^\star\) such that
\(\varepsilon'_{m^\star}\le\bar\varepsilon_{\mathcal I}
\le2\varepsilon'_{m^\star}\).
Because \(\bar\varepsilon_{\mathcal I}\) is fixed under the block
conditioning, this is a fixed comparator base for the master and base
lemmas.  Moreover,
\[
    \varepsilon'_{m^\star}
    +
    \frac{\bar\varepsilon_{\mathcal I}^2}{\varepsilon'_{m^\star}}
    \le
    3\bar\varepsilon_{\mathcal I}.
\]
If \(\bar\varepsilon_{\mathcal I}<1/n\), choose the smallest grid point
\(\varepsilon'_{m^\star}=1/n\). The resulting additive \(O(\sqrt d)\) term
is absorbed by the leading term.  Hence
\[
    \mathbb E_\tau[\Reg(\mathcal I)]
    \le
    \widetilde O\!\left(
        \sqrt{dn\,\Reg_{\rm sq}(n)}
        +
        \sqrt d\,n\,\bar\varepsilon_{\mathcal I}
    \right).
\]
With online Newton step for the bounded linear square-loss class,
$\Reg_{\rm sq}(n)=\widetilde O(d)$, and therefore
\[
    \mathbb E_\tau[\Reg(\mathcal I)]
    \le
    \widetilde O\!\left(
        d\sqrt n
        +
        \sqrt d\,n\,\bar\varepsilon_{\mathcal I}
    \right).
\]
\end{proof}

\begin{remark}[Why the reduction works for adaptive decision sets]
The misspecification analysis of \citet{foster2020adapting} is stated for
oblivious sequences, where the contexts, action sets, and comparator used to
define the misspecification level are fixed before the learner's
randomization. After conditioning on the sequence, the
misspecification-dependent terms in their proof are fixed and hence
predictable. With adaptive decision sets, this need
not hold: a comparator chosen after observing the realized block could depend
on future decision sets, which may themselves depend on later learner
randomization, and hence need not be \(\mathcal G_t\)-measurable at earlier
rounds. Assumption~\ref{ass:adaptive-block-comparator} avoids this issue by
fixing the block comparator and radius before the block randomization.
Therefore the residual-dependent terms used in
\eqref{eq:conditional-importance-weighting} are predictable, and the
SquareCB.Lin+ base and master guarantees can be invoked conditionally on the
block-start information.
\end{remark}
Theorem~\ref{thm:known-path-dynamic-regret} converts the block guarantee of
Lemma~\ref{thm:block-mis-adaptive} into a dynamic-regret bound for
non-stationary linear bandits with general compact decision sets, using the
restarted corralled SquareCB.Lin+ algorithm.
% Theorem~\ref{thm:known-path-dynamic-regret} converts the block guarantee of
% Lemma~\ref{thm:block-mis-adaptive} into a dynamic regret
% bound for the restarted corralled SquareCB.Lin+ algorithm.
\begin{theorem}
\label{thm:known-path-dynamic-regret}
% Suppose the problem setting above holds with \(\|a\|_2\le L\), and suppose
% the assumptions of Lemma~\ref{thm:block-mis-adaptive} hold on every
% restarted block when the block comparator is chosen as the block anchor
% \(\theta_{\mathcal I}^\star=\theta_\tau\).  Assume also that the
% square-loss oracle used by Algorithm~\ref{alg:restart-misspec-adaptive},
% implemented through the weighted-update reduction of
% \citet{foster2020adapting}, satisfies
% \(\Reg_{\rm sq}(n)=\widetilde O(d)\) on every block of length \(n\).
Suppose Assumptions~\ref{ass:subgaussian-noise}
and~\ref{ass:bounded-rewards} hold. Suppose the parameter path
\((\theta_t)_{t=1}^T\) is fixed before interaction begins, has known path
length \(P_T\), and is independent of the learner's randomization. The
decision sets may be adaptive but non-anticipating.
Run Algorithm~\ref{alg:restart-misspec-adaptive} with block length
\[
    \Delta
    =
    \left\lceil
    \min\left\{
        T,
        \max\left\{
            1,
            \left(
                \frac{\sqrt d\,T}{L P_T}
            \right)^{2/3}
        \right\}
    \right\}
    \right\rceil .
\]
Here \((\sqrt d\,T/(LP_T))^{2/3}\) is interpreted as \(+\infty\) when
\(P_T=0\). Then
\[
    \Reg_T
    \le
    \widetilde O\!\left(
        d\sqrt T
        +
        L^{1/3}d^{5/6}T^{2/3}P_T^{1/3}
    \right).
\]
% When \(P_T>0\) and
% \((\sqrt d\,T/(LP_T))^{2/3}\in[1,T]\), this gives
% \(\Reg_T\le
% \widetilde O(L^{1/3}d^{5/6}T^{2/3}P_T^{1/3})\).  For \(P_T=0\), the same
% algorithm with \(\Delta=T\) gives \(\Reg_T\le\widetilde O(d\sqrt T)\).
\end{theorem}

\begin{proof}{Proof of Theorem~\ref{thm:known-path-dynamic-regret}.}
Let \(\mathcal I_i=\{\tau_i,\ldots,\tau_i+n_i-1\}\), with
\(n_i:=|\mathcal I_i|\le\Delta\), denote the restarted blocks.  The last
block may have $n_i<\Delta$.
For each block, set
\[
    P_{\mathcal I_i}
    :=
    \sum_{s=\tau_i+1}^{\tau_i+n_i-1}
    \|\theta_s-\theta_{s-1}\|_2 .
\]
Choosing the block anchor \(\theta_{\mathcal I_i}^\star=\theta_{\tau_i}\),
define the fixed analysis radius
\[
    \bar\varepsilon_{\mathcal I_i}
    :=
    \min\{L P_{\mathcal I_i},B_\varepsilon\},
    \qquad B_\varepsilon:=2LS.
\]
Because the parameter path is oblivious, \(P_{\mathcal I_i}\), and hence
\(\bar\varepsilon_{\mathcal I_i}\), is fixed before the block. Moreover,
for every \(t\in\mathcal I_i\) and \(a\in\mathcal A_t\),
\[
    |\langle a,\theta_t-\theta_{\tau_i}\rangle|
    \le
    L\|\theta_t-\theta_{\tau_i}\|_2
    \le
    L P_{\mathcal I_i},
\]
and also
\[
    |\langle a,\theta_t-\theta_{\tau_i}\rangle|
    \le
    L(\|\theta_t\|_2+\|\theta_{\tau_i}\|_2)
    \le
    B_\varepsilon.
\]
Thus Assumption~\ref{ass:adaptive-block-comparator} holds with
\(\bar\varepsilon_{\mathcal I_i}\). Applying
Lemma~\ref{thm:block-mis-adaptive} to this restarted block, together
with \(\Reg_{\rm sq}(n_i)=\widetilde O(d)\), gives
\[
\begin{aligned}
    \mathbb E_{\tau_i}[\Reg(\mathcal I_i)]
    &\le
    \widetilde O\!\left(
        d\sqrt{n_i}
        +
        \sqrt d\,n_i\,\bar\varepsilon_{\mathcal I_i}
    \right)                                                    \\
    &\le
    \widetilde O\!\left(
        d\sqrt{\Delta}
        +
        L\sqrt d\,\Delta\,P_{\mathcal I_i}
    \right).
\end{aligned}
\]
Let \(N\) be the number of restarted blocks in the partition. Taking
expectations and summing over blocks gives
\[
\begin{aligned}
    \Reg_T
    &\le
    \widetilde O\!\left(
        Nd\sqrt{\Delta}
        +
        L\sqrt d\,\Delta
        \sum_{i=1}^N P_{\mathcal I_i}
    \right)                                                     \\
    &\le
    \widetilde O\!\left(
        \frac{dT}{\sqrt{\Delta}}
        +
        L\sqrt d\,\Delta P_T
    \right).
\end{aligned}
\]
Here we used \(N\le T/\Delta+1\le 2T/\Delta\) and
\(\sum_iP_{\mathcal I_i}\le P_T\). 
% The latter sum includes only
% within-block parameter changes.

If \(P_T=0\), the theorem sets \(\Delta=T\), and the preceding bound gives
\[
    \Reg_T
    \le
    \widetilde O(d\sqrt T).
\]
If \(P_T>0\), let
\[
    \Delta_\star
    :=
    \left(
        \frac{\sqrt d\,T}{L P_T}
    \right)^{2/3}.
\]
If
\(\Delta_\star>T\), using \(\Delta=T\) gives
\(\Reg_T\le\widetilde O(d\sqrt T)\). If \(\Delta_\star<1\), then
\(LP_T>\sqrt d\,T\) implies that
\(L^{1/3}d^{5/6}T^{2/3}P_T^{1/3}\) is at least order \(T\) for \(d\ge1\),
so we use the trivial bounded-regret bound \(\Reg_T\le T\). Hence in all
cases
\[
    \Reg_T
    \le
    \widetilde O\!\left(
        d\sqrt T
        +
        L^{1/3}d^{5/6}T^{2/3}P_T^{1/3}
    \right).
\]
% as claimed.
\end{proof}

\section{K-Armed Contextual Linear Bandits with an Oblivious Adversary}
\label{sec:finite-arm}
In this section we consider the \(K\)-armed contextual linear-bandit
specialization under an oblivious adversary, where
\(\mathcal A_t=\{x_t(i):i\in[K]\}\) at each round.  In this setting, the restarted SupLinUCB algorithm can be analyzed
directly. This uses the observation of \citet{takemura2021parameter} that
SupLinUCB satisfies a misspecification-adaptive regret bound. \ZH{Because
the standard SupLinUCB analysis relies on the stagewise independence property
under contexts fixed before the learner's randomization, our SupLinUCB
guarantee inherits the same oblivious-adversary assumption.} On each block,
parameter drift induces a uniform misspecification term, and the
confidence/elimination analysis absorbs this perturbation without a CORRAL
master.
The logarithmic dependence on the number of arms enters through the quantity
\(\Lambda_K:=1+\log K\), which we use throughout this section.
\begin{assumption}[Oblivious block comparator]
\label{ass:oblivious-block-comparator}
The sequence \((\mathcal A_t,\theta_t)_{t=1}^T\) is fixed before the
learner's randomization.  For each block \(\mathcal I\), choose a comparator
\(\theta_{\mathcal I}^\star\in\Theta\), possibly as a function of the fixed
block sequence, and set
\(\varepsilon_{\mathcal I}:=\varepsilon_{\mathcal I}
(\theta_{\mathcal I}^\star)\).
\end{assumption}

We use the standard SupLinUCB algorithm of
\citet[Algorithm~3]{chu2011contextual}, run freshly on each restarted
block.  On a block \(\mathcal I=\{\tau,\ldots,\tau+n-1\}\), set
the number of stages to \(S=\lceil\log_2 n\rceil\), initialize all stage sample sets
\(\Psi_\tau^\ell=\emptyset\), and use
\(\alpha=\max\{1,\sqrt{\frac12\log(2nK/\delta)}\}\).
\ZH{Here \(\Psi_t^\ell\) denotes the set of rounds in the current block that
have been assigned to stage \(\ell\) before round \(t\). In particular,
\(\Psi_{\tau+n}^\ell\) is the final stage-\(\ell\) sample set for the
block.}
The stage notation
\(\widehat A_\ell(t)\), \(A_t^\ell\), \(s_{t,a}^\ell\), and
\(w_{t,a}^\ell\) is as in \citet{chu2011contextual} and is specified in
Lemma~\ref{lem:suplinucb-perturbed-facts}.
\ZH{At a high level, SupLinUCB maintains stage-specific sample
sets, repeatedly refines the active arm set using upper-confidence
comparisons, and either stops when all surviving arms have small widths or
explores an arm with large stage-\(\ell\) width.}
Lemma~\ref{thm:suplinucb-block-mis} gives the block regret guarantee for
restarted SupLinUCB under uniform block misspecification.
\begin{lemma}
\label{thm:suplinucb-block-mis}
Fix a block $\mathcal I=\{\tau,\ldots,\tau+n-1\}$ and suppose
Assumption~\ref{ass:oblivious-block-comparator} holds on $\mathcal I$.
Assume that Assumption~\ref{ass:subgaussian-noise} holds,
\(\mathcal A_t=\{x_t(i):i\in[K]\}\), \(\|a\|_2\le1\) for all
\(a\in\mathcal A_t\), and \(\|\theta_{\mathcal I}^\star\|_2\le 1\). Run
SupLinUCB freshly on \(\mathcal I\) with confidence parameter
\(\alpha=\max\{1,\sqrt{\frac12\log(2nK/\delta)}\}\).
Then, with probability at least $1-\delta$, conditional on
$\mathcal H_{\tau-1}$,
\[
    \Reg(\mathcal I)
    \le
    \widetilde O\!\left(
        \sqrt{dn}\,\Lambda_K
        +
        \varepsilon_{\mathcal I}n\sqrt{d\Lambda_K}
    \right),
\]
where $\widetilde O(\cdot)$ hides polylogarithmic factors in
$n,d,1/\delta$, but not in \(K\).  Consequently, if per-round regret is
bounded by a universal constant, choosing $\delta=1/n$ gives
\[
    \mathbb E_\tau[\Reg(\mathcal I)]
    \le
    \widetilde O\!\left(
        \sqrt{dn}\,\Lambda_K
        +
        \varepsilon_{\mathcal I}n\sqrt{d\Lambda_K}
    \right).
\]

% This result belongs to the oblivious part of the analysis.  It is not a
% replacement for the fully adaptive-adversary SquareCB/CORRAL block
% guarantee, because the proof below invokes the stagewise independence
% property established for SupLinUCB in the original analysis of
% \citet{chu2011contextual}.
\end{lemma}
\begin{proof}{Proof of Lemma~\ref{thm:suplinucb-block-mis}.}
Condition on \(\mathcal H_{\tau-1}\).  Under
Assumption~\ref{ass:oblivious-block-comparator}, write
\(\mu_t(a)=\langle a,\theta_{\mathcal I}^\star\rangle+\epsilon_t(a)\),
where \(|\epsilon_t(a)|\le\varepsilon_{\mathcal I}\).
\ZH{After conditioning, the block comparator, the contextual action sets,
and the residual functions \(\epsilon_t(\cdot)\) are fixed before the
within-block randomization of SupLinUCB.}
Thus, relative to the fixed comparator \(\theta_{\mathcal I}^\star\), the
block is a realizable contextual linear bandit with an additive
misspecification term that is deterministic after conditioning and uniformly
bounded by \(\varepsilon_{\mathcal I}\).
\ZH{This is the point at which the oblivious block assumption is used: the
misspecification residuals may vary across \(t\) and \(a\), but they do not
adapt to the current randomized action or reward noise.}
By Lemma~\ref{lem:suplinucb-perturbed-facts}, on an
event of conditional probability at least \(1-\delta\), all SupLinUCB
confidence and elimination comparisons satisfy the usual realizable bounds
with an additional perturbation
\[
    \Delta_{\mathcal I}
    =
    C\varepsilon_{\mathcal I}\alpha\sqrt d .
\]
\ZH{Hence the subsequent argument can follow the original SupLinUCB
exploration/elimination counting proof, with each comparison paying an
additional \(O(\Delta_{\mathcal I})\) error.}
SupLinUCB either explores at some stage \(\ell\), in which case the played
round is added to the stage-\(\ell\) sample set, or it stops confidently and
plays from the surviving active set. Let \(\mathcal E_\ell\) be the set of
rounds on which the algorithm explores at stage \(\ell\), and let
\(\mathcal T_{\rm conf}\) be the set of confident-stopping rounds. These
sets partition \(\mathcal I\):
\ZH{each round either explores at exactly one stage or stops confidently, so
every round is counted once.}
Moreover,
\(\mathcal E_\ell=\Psi_{\tau+n}^\ell\), since the restarted algorithm
initializes \(\Psi_\tau^\ell=\emptyset\) and adds a round to \(\Psi^\ell\)
exactly when it explores at stage \(\ell\).
On an
exploration round in \(\mathcal E_\ell\), the active-set part of
Lemma~\ref{lem:suplinucb-perturbed-facts} bounds the instantaneous regret
by \(C2^{-\ell}+CS\Delta_{\mathcal I}\).  On a confident-stopping round,
the same lemma gives \(Cn^{-1/2}+CS\Delta_{\mathcal I}\).  Summing these
two bounds over the partition gives
\ZH{\[
\begin{aligned}
    \Reg(\mathcal I)
    &\le
    \sum_{\ell=1}^{S}
    |\mathcal E_\ell|
    \bigl(C2^{-\ell}+CS\Delta_{\mathcal I}\bigr)
    +
    |\mathcal T_{\rm conf}|
    \bigl(Cn^{-1/2}+CS\Delta_{\mathcal I}\bigr)  \\
    &\le
    C\sum_{\ell=1}^{S}2^{-\ell}|\mathcal E_\ell|
    +
    C\sqrt n
    +
    CS\Delta_{\mathcal I}n,
\end{aligned}
\]}
where the second inequality uses
\ZH{\(|\mathcal T_{\rm conf}|\le n\)} and
\ZH{\(\sum_{\ell=1}^S|\mathcal E_\ell|+|\mathcal T_{\rm conf}|=n\)}.
Since \(\mathcal E_\ell=\Psi_{\tau+n}^\ell\), this yields
\[
\begin{aligned}
    \Reg(\mathcal I)
    &\le
    C\sum_{\ell=1}^{S}2^{-\ell}|\Psi_{\tau+n}^\ell|
    +
    C\sqrt n
    +
    CS\Delta_{\mathcal I}n .
\end{aligned}
\]
Here \(S=\lceil\log_2 n\rceil\).  Lemma~6 of
\citet{chu2011contextual} gives
\[
    |\Psi_{\tau+n}^\ell|
    \le
    5\cdot 2^\ell(1+\alpha^2)
    \sqrt{d|\Psi_{\tau+n}^\ell|}.
\]
\ZH{This invocation is unaffected by misspecification because Lemma~6 of
\citet{chu2011contextual} is a deterministic counting bound for the
stage-\(\ell\) sample set, based on the
elliptical-potential argument and the rule that a round is added to
\(\Psi^\ell\) only when its stage-\(\ell\) width is large. It does not rely
on the linear reward model being well specified, nor on the confidence
intervals being valid for the true rewards.}
\ZH{Multiplying this display by \(2^{-\ell}\) and summing over stages gives
\[
    \sum_{\ell=1}^S2^{-\ell}|\Psi_{\tau+n}^\ell|
    \le
    C(1+\alpha^2)\sqrt d
    \sum_{\ell=1}^S\sqrt{|\Psi_{\tau+n}^\ell|}.
\]
Since the stage sample sets are disjoint and contain at most \(n\) rounds,
Cauchy--Schwarz gives
\[
    \sum_{\ell=1}^S\sqrt{|\Psi_{\tau+n}^\ell|}
    \le
    \sqrt{S\sum_{\ell=1}^S|\Psi_{\tau+n}^\ell|}
    \le
    \sqrt{Sn}.
\]
Therefore,}
\[
    \sum_{\ell=1}^S2^{-\ell}|\Psi_{\tau+n}^\ell|
    \le
    C(1+\alpha^2)\sqrt{dSn}.
\]
Using \(S=O(\log n)\),
\(\alpha^2=O(\log(nK/\delta))\), and
\(\Delta_{\mathcal I}=C\varepsilon_{\mathcal I}\alpha\sqrt d\), we obtain
\[
    \Reg(\mathcal I)
    \le
    \widetilde O\!\left(
        \sqrt{dn}\,\Lambda_K
        +
        \varepsilon_{\mathcal I}n\sqrt{d\Lambda_K}
    \right)
\]
with conditional probability at least \(1-\delta\).  Taking
\(\delta=1/n\) and using bounded per-round regret gives the conditional
expectation bound.
\end{proof}

Theorem~\ref{thm:known-path-karmed-dynamic-regret} converts the SupLinUCB
block guarantee into a known-path-length dynamic regret bound for the
\(K\)-armed contextual case.
\begin{theorem}
\label{thm:known-path-karmed-dynamic-regret}
Consider the \(K\)-armed contextual linear-bandit setting under
Assumptions~\ref{ass:subgaussian-noise},
\ref{ass:path-length}, and~\ref{ass:oblivious-block-comparator}, with
\(\mathcal A_t=\{x_t(i):i\in[K]\}\), \(\|a\|_2\le1\), and
\(\|\theta_t\|_2\le1\). Run SupLinUCB independently on consecutive blocks
of length
\[
    \Delta
    =
    \left\lceil
    \min\left\{
                T,
                \max\left\{
                    1,
                    \left(
                        \frac{\sqrt{\Lambda_K}\,T}{P_T}
                    \right)^{2/3}
                \right\}
            \right\}
    \right\rceil ,
\]
where \((\sqrt{\Lambda_K}T/P_T)^{2/3}\) is interpreted as \(+\infty\)
when \(P_T=0\).
Then
\[
    \Reg_T
    \le
    \widetilde O\!\left(
        \sqrt{dT}\,\Lambda_K
        +
        \sqrt d\,\Lambda_K^{5/6}T^{2/3}P_T^{1/3}
    \right),
\]
where \(\widetilde O(\cdot)\) hides polylogarithmic factors in
\(T,d\), but not in \(K\).
\end{theorem}

\begin{proof}{Proof of Theorem~\ref{thm:known-path-karmed-dynamic-regret}.}
On each block \(\mathcal I_i=\{\tau_i,\ldots,\tau_i+n_i-1\}\), the anchor
\(\theta_{\mathcal I_i}^\star=\theta_{\tau_i}\) gives
\(\varepsilon_{\mathcal I_i}\le P_{\mathcal I_i}\), where $P_{\mathcal I_i}
    :=
    \sum_{s=\tau_i+1}^{\tau_i+n_i-1}
    \|\theta_s-\theta_{s-1}\|_2 $. The norm bounds also give
\(\|\theta_{\mathcal I_i}^\star\|_2\le1\) and uniformly bounded per-round
regret.
% \[
%     P_{\mathcal I_i}
%     :=
%     \sum_{s=\tau_i+1}^{\tau_i+n_i-1}
%     \|\theta_s-\theta_{s-1}\|_2 .
% \]
Lemma~\ref{thm:suplinucb-block-mis} therefore gives
\[
    \mathbb E_{\tau_i}[\Reg(\mathcal I_i)]
    \le
    \widetilde O\!\left(
        \sqrt{dn_i}\,\Lambda_K
        +
        \sqrt{d\Lambda_K}\,n_iP_{\mathcal I_i}
    \right).
\]
Summing over the restarted blocks gives the tradeoff
\[
    \Reg_T
    \le
    \widetilde O\!\left(
        \frac{\sqrt d\,\Lambda_K T}{\sqrt\Delta}
        +
        \sqrt{d\Lambda_K}\,\Delta P_T
    \right).
\]
Optimizing this display with the chosen \(\Delta\) gives the stated bound.
The endpoint cases \(\Delta=T\) and \(\Delta=1\) are handled as in
Theorem~\ref{thm:known-path-dynamic-regret}.
\end{proof}

The following proposition records a lower bound for non-stationary
\(K\)-armed contextual linear bandits. 
A proof sketch is deferred to Appendix~\ref{app:finite-k-lower-bound-proof}.
\begin{proposition}
\label{prop:finite-k-lower-bound}
Consider \(K\)-armed contextual linear bandits with \(K\ge2\) actions per
round, \(\|x_{t,a}\|_2\le1\), \(\|\theta_t\|_2\le1\), and path-length $\sum_{t=1}^{T-1}\|\theta_{t+1}-\theta_t\|_2\le P_T$.
% budget
% \[
%     \sum_{t=1}^{T-1}\|\theta_{t+1}-\theta_t\|_2\le P_T .
% \]
Let \(\Gamma_{d,K}\) be the maximum of \((d\min\{K,d\})^{1/6}\) and
\((\min\{d,\lfloor\log_2K\rfloor\})^{2/3}\).  Then the minimax dynamic
regret is at least
\(\Omega(\sqrt{dT}\vee P_T^{1/3}T^{2/3}\Gamma_{d,K})\).
The lower bound holds for obliviously chosen contexts and parameter
sequences.
\end{proposition}

Proposition~\ref{prop:finite-k-lower-bound} shows that the restarted
SupLinUCB bound is optimal in its dependence on \(T\) and \(P_T\). The
remaining gap is in the dimension/action-set dependence: the upper bound has
non-stationary coefficient \(\sqrt d\,\Lambda_K^{5/6}\), whereas the lower
bound gives \(\Gamma_{d,K}\). We leave closing this dimension/action-set
dependence gap for future work.

\section{Conclusion}
\label{sec:conc}
% We studied non-stationary linear bandits with time-varying decision sets
% through a misspecification-reduction viewpoint.  After restarting the
% learner on blocks, the local parameter drift can be treated as blockwise
% linear misspecification.  This yields dynamic-regret guarantees for general
% linear bandits with adaptive non-anticipating decision sets and an
% oblivious parameter path, as well as for finite-arm contextual linear
% bandits with at most \(K\) actions per round under an oblivious adversary.

% Several questions remain open.  First, the general linear-bandit guarantee
% \(\widetilde O(d^{5/6}T^{2/3}P_T^{1/3})\) has a factor
% \(\widetilde O(d^{1/6})\) gap relative to the
% \(\Omega(d^{2/3}T^{2/3}P_T^{1/3})\) lower bound of
% \citet{cheung2022hedging}.  Closing this dimension gap, or proving that it
% is inherent to the misspecification-reduction approach, is an important
% theoretical question.  Second, the general linear-bandit result uses a
% CORRAL-style master to aggregate base algorithms over candidate
% misspecification radii.  It would be useful to understand whether this
% aggregation layer is necessary, or whether one can design a single base
% algorithm that adapts directly to the unknown block misspecification level,
% in the spirit of adaptive guarantees such as \citet{hu2025learning}.
{\color{ZHblue}
We developed a blockwise misspecification approach to
non-stationary linear bandits with a known path-length budget.
The resulting guarantees cover adaptive non-anticipating
compact decision sets under an oblivious parameter path,
and finite-action contextual settings under an oblivious
environment.
They match the non-stationary lower-bound exponents in
$T$ and $P_T$, up to logarithmic factors.
}

Two important questions remain open.  First, the general linear-bandit guarantee
\(\widetilde O(d^{5/6}T^{2/3}P_T^{1/3})\) has a factor
\(\widetilde O(d^{1/6})\) gap relative to the
\(\Omega(d^{2/3}T^{2/3}P_T^{1/3})\) lower bound of
\citet{cheung2022hedging}.  Closing this dimension gap is an important
theoretical question.  Second, for non-stationary linear bandits with
general compact decision sets, it remains unclear whether the CORRAL-style
aggregation layer is necessary, or whether a single base algorithm can adapt
directly to the unknown block misspecification level.
% , in the spirit of
% adaptive guarantees such as \citet{hu2025learning}.

\section*{Acknowledgements}
We thank Feng Ruan, Hongfan Wu, Yinyu Ye, and Peng Zhao for helpful discussions at different stages of this work.

% \newpage
\bibliography{refs}

\begin{APPENDICES}
\renewcommand{\theHsection}{appendix.\arabic{section}}

\section{Technical Lemmas}
\label{app:technical}

\begin{lemma}[\citealt{foster2020adapting}]
\label{lem:hedged-tsallis-master}
Consider a block \(\mathcal I=\{\tau,\ldots,\tau+n-1\}\) with \(M\)
base learners and master losses \(\tilde L_{t,m}\in[0,2]\). Suppose the
master runs the $\left(1/2,\frac{3}{2}\sqrt{dn\,\Reg_{\rm sq}(n)}\right)$-hedged Tsallis-INF algorithm of
\citet{foster2020adapting} on this block. Then, for every fixed base
\(m\in[M]\),
% \[
%     \mathbb E_\tau\!\left[
%         \sum_{t\in\mathcal I}
%         \bigl(\tilde L_{t,M_t}-\tilde L_{t,m}\bigr)
%     \right]
%     \le
%     4\sqrt{2Mn}
%     +
%     2\sqrt M\,R
%     -
%     R\,\mathbb E_\tau[\sqrt{\rho_{\mathcal I,m}}].
% \]
% In particular, for the logarithmic grid used in this section and
% \[
%     R=\frac{3}{2}\sqrt{dn\,\Reg_{\rm sq}(n)},
% \]
the master regret is bounded by
\[
\begin{aligned}
    &\mathbb E_\tau\!\left[
        \sum_{t\in\mathcal I}
        \bigl(\tilde L_{t,M_t}-\tilde L_{t,m}\bigr)
    \right]\\
    \le&
    \widetilde O\!\left(\sqrt{dn\,\Reg_{\rm sq}(n)}\right)\\
    &\quad -
    \frac{3}{2}
    \mathbb E_\tau[\sqrt{\rho_{\mathcal I,m}}]
    \sqrt{dn\,\Reg_{\rm sq}(n)} .
\end{aligned}
\]
\end{lemma}

% \begin{proof}
% This is Corollary~2 of \citet{foster2020adapting}, applied conditionally
% on the block history \(\mathcal H_{\tau-1}\) and with the block reindexed
% as a length-\(n\) adversarial multi-armed bandit problem. For
% \(\alpha=1/2\), that result gives
% \[
%     \mathbb E_\tau\!\left[
%         \sum_{t\in\mathcal I}
%         \bigl(\tilde L_{t,M_t}-\tilde L_{t,m}\bigr)
%     \right]
%     \le
%     4\sqrt{2Mn}
%     +
%     \mathbb E_\tau\!\left[
%         \bigl(
%             \min\{2,2\log(\rho_{\mathcal I,m})\}\sqrt M
%             -
%             \rho_{\mathcal I,m}
%         \bigr)R
%     \right].
% \]
% Since \(\rho_{\mathcal I,m}\ge1\), the second term is at most
% \[
%     2\sqrt M\,R
%     -
%     R\,\mathbb E_\tau[\sqrt{\rho_{\mathcal I,m}}],
% \]
% which proves the first display. The second display follows by substituting
% \(R=(3/2)\sqrt{dn\,\Reg_{\rm sq}(n)}\) and absorbing the logarithmic-grid
% dependence on \(M\) into the \(\widetilde O(\cdot)\) notation.
% \end{proof}

\begin{lemma}[\citealt{foster2020adapting}]
\label{lem:logdet-one-step}
Let \(\mathcal A\subset\mathbb R^d\), \(\gamma>0\), and
\(\widehat\beta\in\mathbb R^d\). For \(p\in\Delta(\mathcal A)\), write
\(\bar a_p:=\mathbb E_{a\sim p}[a]\) and
\(H_p:=\mathbb E_{a\sim p}[aa^\top]\).
Let \(p\) be any member of
\(\operatorname{logdet\text{-}barrier}
(\widehat\beta,\gamma;\mathcal A)\), as defined in
\eqref{eq:logdet-barrier-definition}.
Then, for every \(a^\star\in\mathcal A\) and every
\(\beta\in\mathbb R^d\),
\[
\begin{aligned}
    \langle \bar a_p-a^\star,\beta\rangle
    \le&
    \frac{\dim(\mathcal A)}{\gamma}
    +
    \frac{\gamma}{4}
    \|\widehat\beta-\beta\|_{H_p}^2 \\
    \le&
    \frac{d}{\gamma}
    +
    \frac{\gamma}{4}
    \|\widehat\beta-\beta\|_{H_p}^2 .
\end{aligned}
\]
% In particular, since \(\dim(\mathcal A)\le d\),
% \[
%     \langle \bar a_p-a^\star,\beta\rangle
%     \le
%     \frac{d}{\gamma}
%     +
%     \frac{\gamma}{4}
%     \|\widehat\beta-\beta\|_{H_p}^2 .
% \]
\end{lemma}

\begin{lemma}[Perturbed SupLinUCB facts]
\label{lem:suplinucb-perturbed-facts}
Fix a block $\mathcal I=\{\tau,\ldots,\tau+n-1\}$ and condition on
$\mathcal H_{\tau-1}$.  Suppose Assumption~\ref{ass:oblivious-block-comparator}
holds, and write
\(\mu_t(a)=\langle a,\theta_{\mathcal I}^\star\rangle+\epsilon_t(a)\),
where \(|\epsilon_t(a)|\le \varepsilon_{\mathcal I}\).
Let SupLinUCB be run freshly on $\mathcal I$. Then, conditional on
$\mathcal H_{\tau-1}$, with probability at least $1-\delta$, the following
two statements hold simultaneously, with the stagewise notation defined in
the proof.
First, whenever the algorithm uses the stage-$\ell$ scores
$\hat r_{t,a}^\ell+w_{t,a}^\ell$ either to select an arm confidently or to
eliminate arms, the following bound holds for every active arm involved:
\begin{equation}
\label{eq:perturbed-facts-confidence}
    |\hat r_{t,a}^\ell-\mu_t(a)|
    \le
    2w_{t,a}^\ell+\Delta_{\mathcal I},
    \qquad
    \Delta_{\mathcal I}
    :=
    C\varepsilon_{\mathcal I}\alpha\sqrt d .
\end{equation}
Second, for every round $t\in\mathcal I$ and every stage $\ell$ reached on
that round,
\begin{equation}
\label{eq:perturbed-facts-active}
    \mu_t(a_t^\star)-\mu_t(a)
    \le
    C2^{-\ell}+C\ell\Delta_{\mathcal I},
    \qquad
    a\in\widehat A_\ell(t).
\end{equation}
Moreover, on a confident stopping round,
\begin{equation}
\label{eq:perturbed-facts-confident}
    \mu_t(a_t^\star)-\mu_t(a_t)
    \le
    Cn^{-1/2}+CS\Delta_{\mathcal I}.
\end{equation}
Here \(S:=\lceil\log_2 n\rceil\).
\end{lemma}

\begin{proof}{Proof of Lemma~\ref{lem:suplinucb-perturbed-facts}.}
Under Assumption~\ref{ass:oblivious-block-comparator}, the block sequence
is fixed before the learner's randomization.  Therefore the original
SupLinUCB construction satisfies the stagewise independence property
established in Lemma~4 of \citet{chu2011contextual}.  We use this property
below.

Let \(S=\lceil\log_2 n\rceil\). For each stage \(\ell\in[S]\), let
\(\Psi_t^\ell\) collect the stage-\(\ell\) samples from the current block
available before round \(t\). Define
\[
    A_t^\ell
    =
    I_d+\sum_{\rho\in\Psi_t^\ell}a_\rho a_\rho^\top,
    \qquad
    s_{t,a}^\ell
    =
    \sqrt{a^\top(A_t^\ell)^{-1}a}.
\]
Set \(w_{t,a}^\ell:=\alpha s_{t,a}^\ell\).
Let
\[
    \widehat\theta_t^\ell
    =
    (A_t^\ell)^{-1}
    \sum_{\rho\in\Psi_t^\ell}a_\rho r_\rho,
    \qquad
    \hat r_{t,a}^\ell
    =
    \langle a,\widehat\theta_t^\ell\rangle .
\]
Let $\widehat A_\ell(t)$ be the active arm set when round $t$ reaches
stage $\ell$, and let
$a_t^\star\in\argmax_{a\in\mathcal A_t}\mu_t(a)$.

The rest of the proof is the standard SupLinUCB proof with two
deterministic misspecification terms.  For a fixed stage $\ell$, let
$D_t^\ell$ be the
design matrix formed by the samples in $\Psi_t^\ell$.  Since
\[
    r_\rho
    =
    \langle a_\rho,\theta_{\mathcal I}^\star\rangle
    +
    \epsilon_\rho(a_\rho)
    +
    \eta_\rho ,
\]
where
\(\epsilon=(\epsilon_\rho(a_\rho))_{\rho\in\Psi_t^\ell}\)
denotes the vector of signed misspecification residuals.
The usual BaseLinUCB decomposition gives the realizable confidence term
$(\alpha+1)s_{t,a}^\ell$.  Since
$w_{t,a}^\ell=\alpha s_{t,a}^\ell$ and $\alpha\ge1$, this term is at most
$2w_{t,a}^\ell$.  Hence
\[
\begin{aligned}
    |\hat r_{t,a}^\ell-\mu_t(a)|
    &\le
    2w_{t,a}^\ell
    +
    \left|
        a^\top(A_t^\ell)^{-1}(D_t^\ell)^\top\epsilon
    \right|
    +
    \varepsilon_{\mathcal I}                                      \\
    &\le
    2w_{t,a}^\ell
    +
    \varepsilon_{\mathcal I}\sqrt{|\Psi_t^\ell|}\,s_{t,a}^\ell
    +
    \varepsilon_{\mathcal I}.
\end{aligned}
\]
The first term is the standard confidence term in
\citet{chu2011contextual}. The last two terms are the misspecification
contributions.  The final $\varepsilon_{\mathcal I}$ term will be absorbed
into $\Delta_{\mathcal I}$.
By Lemma~6 of \citet{chu2011contextual},
\[
    \sqrt{|\Psi_{\tau+n}^\ell|}
    \le
    5\cdot 2^\ell(1+\alpha^2)\sqrt d .
\]
Whenever the algorithm performs a stage-$\ell$ elimination comparison,
$w_{t,a}^\ell\le 2^{-\ell}$ for all active arms, and hence
\[
    s_{t,a}^\ell\le \frac{2^{-\ell}}{\alpha}.
\]
By definition of the stage sample sets,
\(\Psi_t^\ell\subseteq\Psi_{\tau+n}^\ell\) for all \(t\le\tau+n\).
Therefore
\[
    \sqrt{|\Psi_t^\ell|}\,s_{t,a}^\ell
    \le
    \sqrt{|\Psi_{\tau+n}^\ell|}\frac{2^{-\ell}}{\alpha}
    \le
    5\frac{1+\alpha^2}{\alpha}\sqrt d
    \le
    10\alpha\sqrt d .
\]
The confident-stopping case is analogous, using
$w_{t,a}^\ell\le n^{-1/2}$.  Therefore the confidence inequality becomes
\[
    |\hat r_{t,a}^\ell-\mu_t(a)|
    \le
    2w_{t,a}^\ell+\Delta_{\mathcal I},
    \qquad
    \Delta_{\mathcal I}
    =
    C\varepsilon_{\mathcal I}\alpha\sqrt d .
\]

It remains to justify the active-set statement. The argument follows the
active-set induction in Lemma~5 of \citet{chu2011contextual}, except that
each use of the realizable confidence event is replaced by the perturbed
confidence bound \eqref{eq:perturbed-facts-confidence}, adding an
\(O(\Delta_{\mathcal I})\) error per stage. Suppose stage
$\ell$ performs an elimination step, so all active arms have
$w_{t,a}^\ell\le 2^{-\ell}$.  For any two active arms $a,b$, the perturbed
confidence bound implies
\[
\begin{aligned}
    &\mu_t(b)-\mu_t(a)\\
    \le&
    \bigl(\hat r_{t,b}^\ell+w_{t,b}^\ell\bigr)-
    \bigl(\hat r_{t,a}^\ell+w_{t,a}^\ell\bigr)
    +
    C2^{-\ell}+C\Delta_{\mathcal I}.
\end{aligned}
\]
Thus, if $a$ survives the elimination test, then its true mean is below the
best active benchmark by at most $C2^{-\ell}+C\Delta_{\mathcal I}$.
Inductively, take the optimal arm as the benchmark as long as it remains
active. If it is eliminated, replace it by the surviving arm whose
empirical upper confidence value caused the elimination.  Each stage can
increase the benchmark's suboptimality by only $O(\Delta_{\mathcal I})$,
so after $\ell$ stages every active arm satisfies
\eqref{eq:perturbed-facts-active}.  There are at most
$S=\lceil\log_2 n\rceil$ stages.  On a confident stopping round all active
arms have width at most $n^{-1/2}$, and comparing the selected arm with the
same benchmark gives \eqref{eq:perturbed-facts-confident}.
\end{proof}

\section{Proof Sketch of Proposition~\ref{prop:finite-k-lower-bound}}
\label{app:finite-k-lower-bound-proof}

% \begin{proof}[Proof sketch]
The stationary term \(\Omega(\sqrt{dT})\) is the \(K\)-armed contextual
linear-bandit lower bound of \citet{chu2011contextual}.  For the
non-stationary term, split the horizon into lower-bound epochs of length
\(H\).  In each epoch, instantiate an independent stationary \(K\)-armed
hard instance.  We assume \(P_T>0\). When \(P_T=0\), the stationary term
already gives the stated bound.

For one construction, use a grouping argument in the spirit of the
stationary lower bound of \citet[Section~6]{chu2011contextual}.  Let
\(m:=\min\{K,d\}\).  The per-epoch regret and the between-epoch parameter
displacement must be calibrated together: the gap size is chosen at the
standard indistinguishability scale used in the stationary lower bound.
With this calibration, the construction gives per-epoch regret
\(\Omega(\sqrt{dH})\) and between-epoch displacement of order
\(d/\sqrt{mH}\).  Thus the path-length budget is respected whenever
\[
    \frac{T}{H}\frac{d}{\sqrt{mH}}
    \le
    cP_T
\]
for a sufficiently small universal constant \(c>0\).  Choose
\[
    H
    =
    \left\lceil
        C\left(
            \frac{dT}{\sqrt m\,P_T}
        \right)^{2/3}
    \right\rceil
\]
for a sufficiently large universal constant \(C>0\).  Then the total regret
over the epochs is at least
\[
    \frac{T}{H}\sqrt{dH}
    \ge
    c'(dm)^{1/6}P_T^{1/3}T^{2/3}
\]
for a universal constant \(c'>0\).

For the hypercube construction, let
\(r:=\min\{d,\lfloor\log_2K\rfloor\}\) and use the finite action set
\(\{v/\sqrt r:v\in\{\pm1\}^r\}\) embedded in \(\mathbb R^d\).  By the
standard \(\Omega(r\sqrt H)\) lower bound for \(r\)-dimensional stochastic
linear bandits with rich action sets
\citep{dani2008stochastic,lattimore2020bandit}, this hypercube instance
has stationary epoch regret \(\Omega(r\sqrt H)\).  The path-length budget is
respected whenever
\[
    \frac{T}{H}\frac{r}{\sqrt H}
    \le
    cP_T .
\]
Choosing
\[
    H
    =
    \left\lceil
        C\left(\frac{rT}{P_T}\right)^{2/3}
    \right\rceil
\]
therefore gives total regret at least
\[
    \frac{T}{H}r\sqrt H
    \ge
    c''r^{2/3}P_T^{1/3}T^{2/3}
\]
for a universal constant \(c''>0\).
Taking the larger of the two constructions gives the claimed bound.
% \end{proof}

\end{APPENDICES}

%%%%%%%%%%%%%%%%%
\end{document}